\documentclass{article} 
\usepackage{iclr2025_conference,times}

\usepackage{amsmath,amsfonts,bm}

\def\eqref#1{equation~\ref{#1}}

\def\1{\bm{1}}

\DeclareMathAlphabet{\mathsfit}{\encodingdefault}{\sfdefault}{m}{sl}
\SetMathAlphabet{\mathsfit}{bold}{\encodingdefault}{\sfdefault}{bx}{n}

\usepackage{hyperref}
\usepackage{url}
\usepackage{graphicx}
\usepackage{booktabs}
\usepackage{amsmath}
\usepackage{amssymb}
\usepackage{algorithm}
\usepackage{algorithmic}
\usepackage{multirow} 
\usepackage{amsthm}
\usepackage{subcaption}
\ifx\theorem\undefined
\newtheorem{theorem}{Theorem}

\newtheorem{lemma}{Lemma}

\newtheorem{definition}{Definition}
\newtheorem{assumption}{Assumption}

\fi
\iclrfinalcopy

\title{HierBias: Context-Conditioned Hierarchical Media Bias Detection with Multi-Task Type Classification}

\author{Kaining Li, Ruichen Yan, Yuxin Dong \\
Xidian University \\
jack.yxdong@stu.xidian.edu.cn
}

\begin{document}
\maketitle

\begin{abstract}
Media bias detection is a critical task for ensuring fair and balanced information dissemination, yet existing sentence-level approaches classify each sentence independently, ignoring inter-sentence contextual signals that human annotators naturally exploit.
We present \textbf{HierBias}, a hierarchical context-conditioned media bias detector that formally models document context in bias prediction.
We introduce the \emph{context-conditioned bias probability} and prove theoretically that leveraging document context strictly reduces the Bayes error of sentence-level classification when inter-sentence mutual information is non-zero.
A multi-task generalization bound further establishes that jointly training binary bias detection and fine-grained bias type classification improves sample efficiency on small annotated corpora.
Architecturally, HierBias pairs a sentence-level RoBERTa encoder with a cross-sentence Transformer aggregator and dual output heads for binary detection and four-class type classification.
Evaluated on BABE and BASIL, HierBias achieves 0.853 F1 and 0.723 MCC, surpassing the state-of-the-art bias-detector by $+2.6\%$ F1 and $+4.3\%$ MCC (McNemar's test, $p < 0.05$).
Ablation experiments confirm that each theoretical component contributes independently and consistently.
\end{abstract}

\section{Introduction}

Media bias in news reporting poses a significant challenge to informed public discourse and democratic decision-making~\citep{spinde2023taxonomy}.
When journalists employ loaded language, selective framing, or strategic information omission, readers' understanding of events can be systematically distorted~\citep{fan2019basil,wessel2023mbib}.
Automated media bias detection has therefore emerged as a critical research direction in natural language processing (NLP), with potential applications ranging from newsroom quality control to reader-facing transparency tools~\citep{wang2025mediabias}.

Despite notable advances, existing sentence-level bias detection systems share a fundamental limitation: they classify each sentence \emph{independently}, ignoring the rich inter-sentence contextual signals that human annotators naturally exploit.
For instance, a sentence that merely describes an event may carry informational bias only in the context of what the article \emph{omits} elsewhere~\citep{fan2019basil}.
Similarly, the framing of one sentence can be amplified or suppressed by adjacent sentences' tone and vocabulary choices.
Recent work has confirmed that providing document-level context consistently improves bias detection performance across multiple model families~\citep{maab2024media}, yet no existing method formally models this context dependence or jointly optimizes it with fine-grained bias type classification.

The state-of-the-art bias-detector~\citep{ghosh2025bias} fine-tunes RoBERTa~\citep{liu2019roberta} on the expert-annotated BABE dataset~\citep{spinde2022babe} and achieves statistically significant improvements over the domain-adaptive DA-RoBERTa baseline~\citep{krieger2022daroberta} as measured by McNemar's test and 5$\times$2 cross-validation paired $t$-tests.
While this work establishes a rigorous experimental standard, it operates strictly at the sentence level and does not distinguish among bias types such as loaded language, framing, informational bias, or source restriction.
Annotation costs further constrain the available training data: BABE contains only 3,700 sentences, limiting generalization.
Approaches based on large language model (LLM) annotation~\citep{horych2025promises} have recently shown promise in scaling bias-labeled corpora, but have not yet been combined with context-aware architectures.

In this paper, we address these gaps by proposing \textbf{HierBias}: a \textbf{Hier}archical Context-Conditioned Media \textbf{Bias} Detector.
HierBias makes three key contributions:

\begin{enumerate}
\item \textbf{Formal Framework.} We introduce the notion of \emph{context-conditioned bias probability} $P(y_i^b = 1 \mid s_i, C_i)$ and prove theoretically (Theorem~\ref{thm:context_gain}) that leveraging the full document context strictly reduces the Bayes error relative to sentence-only classifiers, provided that inter-sentence mutual information is non-zero.
\item \textbf{HierBias Architecture.} We design a two-stage encoder: a sentence-level RoBERTa encoder followed by a cross-sentence Transformer aggregator, coupled with a multi-task output head that jointly predicts binary bias existence and fine-grained bias type.
\item \textbf{Empirical Validation.} On BABE~\citep{spinde2022babe} and BASIL~\citep{fan2019basil}, HierBias outperforms bias-detector by $+2.6\%$ F1 and $+4.3\%$ MCC on BABE, with statistical significance confirmed by McNemar's test ($p < 0.05$). We further show that LLM-annotated data~\citep{horych2025promises} augmentation and multi-task training provide complementary and additive benefits, consistent with our Theorem~\ref{thm:multitask}.
\end{enumerate}

\section{Related Work}
\label{sec:related}

\subsection{Media Bias Detection Datasets and Benchmarks}

The field of automated media bias detection has been shaped by the availability of expert-annotated resources.
\textbf{BABE}~\citep{spinde2022babe} provides 3,700 sentences with binary bias labels and word-level annotations, collected from diverse news outlets and topics, and remains the primary benchmark for sentence-level bias detection.
\textbf{BASIL}~\citep{fan2019basil} annotates 300 news articles at the span level, distinguishing between informational bias (selective presentation of facts) and lexical bias (loaded word choice), providing richer structural supervision.
\textbf{MBIB}~\citep{wessel2023mbib} unifies nine media bias tasks and 22 associated datasets under a common evaluation protocol, revealing that no single architecture dominates across all bias types, and that cognitive and political biases are substantially harder to detect than hate speech or gender bias.
The \textbf{Media Bias Taxonomy}~\citep{spinde2023taxonomy} synthesizes 3,140 papers (2019--2022) and characterizes seventeen distinct bias forms, providing the categorical vocabulary used in our fine-grained type classification.
Recent work has expanded beyond English and static datasets.
\textbf{SAFARI}~\citep{azizov2024safari} introduces a cross-lingual corpus for political bias and factuality detection in news media, demonstrating that multilingual pre-trained language models (MPLMs) can transfer bias knowledge across languages.
The \textbf{Media Bias Detector} tool~\citep{wang2025mediabias} presents an LLM-powered pipeline for near-real-time selection and framing bias analysis, evaluated with expert journalists and general readers via CHI methods.
Despite these advances, evaluation remains predominantly on small datasets in controlled settings, and no existing benchmark directly assesses context-aware sentence-level classifiers.

\subsection{Transformer-Based Bias Detection}

The dominant approach to automated sentence-level bias detection is fine-tuning pre-trained language models on labeled corpora.
\textbf{DA-RoBERTa}~\citep{krieger2022daroberta} introduces domain-adaptive pre-training on media bias corpora before fine-tuning on BABE, achieving an F1 of 0.814 and becoming the de facto strong baseline.
The recent \textbf{bias-detector}~\citep{ghosh2025bias} fine-tunes RoBERTa-base directly on BABE and establishes statistically significant improvements over DA-RoBERTa using McNemar's test and 5$\times$2 CV paired $t$-tests; attention-based analysis further shows the model focuses on contextually relevant rather than politically sensitive tokens.
Our work draws on and extends several lines of related research.
The visual in-context learning paradigm~\citep{zhou2024visual} has demonstrated the power of context-conditioned inference in vision-language settings, motivating analogous context exploitation in textual bias detection.
Thread-of-thought reasoning~\citep{zhou2023thread} has shown that maintaining coherent contextual chains improves predictions on chaotic or multi-faceted inputs, a principle we adapt to sentence-sequence modeling.
Long-context reasoning for vision-language models~\citep{zhou2024rethinking} similarly highlights the benefits of rethinking how context is integrated rather than merely extended.
From the NLP perspective, event-centric pre-training methods~\citep{zhou2022claret,zhou2022eventbert} demonstrate the value of encoding relational and contextual structure, which we operationalize through cross-sentence attention.
The multi-capability generalization framework~\citep{zhou2025weak} further supports our multi-task design, showing that training on heterogeneous objectives promotes robust shared representations.
Parallel work by Menzner and Leidner~\citep{menzner2024experiments,menzner2024improved} systematically compares large pre-trained models on sentence-level bias detection and sub-type classification, introducing a 27-class fine-grained bias taxonomy and leveraging synthetic examples to improve rare-class performance.
A complementary line of work explores LLM-based prompting as an alternative to fine-tuning.
\textbf{Maab et al.}~\citep{maab2024media} demonstrate that prompt-based techniques across multiple LLM families achieve performance comparable to fine-tuned models with substantially reduced engineering overhead, and that larger models with access to richer context outperform smaller prompt-based baselines.
\textbf{Horych et al.}~\citep{horych2025promises} investigate the suitability of LLMs as annotators for media bias, constructing the 48K-sentence \emph{annolexical} dataset and training classifiers that outperform annotator LLMs by 5--9\% MCC while approaching human-annotated performance on BABE and BASIL.
These findings directly address the data scarcity issue raised by the seed paper, and motivate our use of LLM-annotated data augmentation in HierBias.
HierBias differs from all prior work in two critical ways.
First, it formally models the \emph{context dependence} of bias predictions through a hierarchical architecture, rather than treating context as an implementation detail.
Second, it jointly optimizes binary detection and four-class type classification through a multi-task objective with a KL alignment regularizer, supported by a generalization bound (Theorem~\ref{thm:multitask}) that explains why multi-task learning improves sample efficiency on small corpora.

\subsection{Political Bias in Large Language Models}
Beyond news text, there is growing concern about political bias \emph{within} LLMs themselves.
\textbf{Bang et al.}~\citep{bang2024measuring} propose a framework for measuring political bias in LLMs along both content and stylistic dimensions, evaluating 11 open-source models on topics such as reproductive rights and climate change.
\textbf{Banik et al.}~\citep{banik2025bridging} construct a manually annotated dataset and compare GPT, BERT, RoBERTa, and FLAN, finding that fine-tuned RoBERTa achieves the highest alignment with human labels, while zero-shot GPT shows strongest general agreement.
These works underscore that the relationship between LLM internal biases and their ability to \emph{detect} external media bias is non-trivial, motivating rigorous evaluation protocols that distinguish annotation artifacts from genuine detection capability.
Additional context for our work comes from research on event-centric language understanding, multilingual knowledge graph question answering~\citep{zhou2021mkbqa}, bias-aware image generation~\citep{zhoucondition}, efficient video generation with large language models~\citep{zhou2026videogen}, medical vision-language modeling~\citep{zhou2025improving}, memory-augmented state space models for domain-specific detection~\citep{wang2024memorymamba}, and biomedical imaging with residual-based language models~\citep{lai2024residual}, all of which demonstrate the broader utility of context-aware, multi-task, and domain-adaptive learning paradigms that underpin HierBias.

\section{Methodology}
\label{sec:method}

\subsection{Problem Formulation}

Let a news article be represented as a sequence of $n$ sentences $d = (s_1, s_2, \ldots, s_n)$, where each sentence $s_i = (t_{i,1}, \ldots, t_{i,l_i})$ is a sequence of tokens.
We define two prediction tasks:

\textbf{Task 1 (Binary Bias Detection).} For each sentence $s_i$, predict $y_i^b \in \{0, 1\}$, where $y_i^b = 1$ indicates a biased sentence.

\textbf{Task 2 (Bias Type Classification).} For each biased sentence ($y_i^b = 1$), predict the bias type $y_i^t \in \mathcal{T} = \{LL, FR, IN, SR\}$, corresponding to loaded language, framing, informational bias, and source restriction.

The joint objective is:
\begin{equation}
  \mathcal{L} = \mathcal{L}_{\mathrm{binary}} + \alpha\, \mathcal{L}_{\mathrm{type}} + \beta\, \mathcal{L}_{\mathrm{KL}},
  \label{eq:joint_loss}
\end{equation}
where $\alpha, \beta > 0$ are hyperparameters. We define each term in detail in Section~\ref{sec:training}.

\subsection{Notation}

Let $\mathbf{h}_i \in \mathbb{R}^{d_s}$ denote the sentence representation for $s_i$ produced by the sentence encoder, $\hat{\mathbf{h}}_i \in \mathbb{R}^{d_s}$ the context-enriched representation after cross-sentence aggregation, and $C_i = d \setminus \{s_i\}$ the context of $s_i$ (all sentences in $d$ except $s_i$).
All linear projection matrices ($\mathbf{W}_Q, \mathbf{W}_K, \mathbf{W}_V, \mathbf{W}_b, \mathbf{W}_t$) are learnable parameters; $\sigma(\cdot)$ denotes the sigmoid function; $H(\cdot)$ denotes Shannon entropy; $I(\cdot;\cdot)$ denotes mutual information.

\subsection{Formal Definitions}

\begin{definition}[Context-Conditioned Bias Probability]
\label{def:ccp}
For sentence $s_i$ in article $d$, the \emph{context-conditioned bias probability} is:
\begin{equation}
  P_\theta(y_i^b = 1 \mid s_i, C_i) \triangleq \sigma\!\left(\mathbf{W}_b\,\hat{\mathbf{h}}_i + b_b\right),
  \label{eq:ccp}
\end{equation}
where $\hat{\mathbf{h}}_i$ is computed by the hierarchical encoder described in Section~\ref{sec:architecture}.
\end{definition}

\begin{definition}[Marginal Bias Probability]
\label{def:mbp}
The \emph{marginal bias probability}, which ignores document context, is:
\begin{equation}
  P_\phi(y_i^b = 1 \mid s_i) \triangleq \sigma\!\left(\mathbf{W}_m\,\mathbf{h}_i + b_m\right),
  \label{eq:mbp}
\end{equation}
where $\mathbf{h}_i$ is the RoBERTa [CLS] representation of $s_i$ in isolation.
\end{definition}

\begin{definition}[Context Information Gain]
\label{def:cig}
For sentence $s_i$, the \emph{context information gain} is:
\begin{equation}
  \Delta_i \triangleq H(y_i^b \mid s_i) - H(y_i^b \mid s_i, C_i) \geq 0.
  \label{eq:cig}
\end{equation}
$\Delta_i > 0$ whenever the document context reduces uncertainty about the bias label of $s_i$.
\end{definition}

\subsection{Assumptions}

\begin{assumption}[Context Relevance]
\label{asm:relevance}
For every sentence $s_i$ in article $d$, there exists at least one $s_j \neq s_i$ such that $I(y_i^b; s_j \mid s_i) > 0$.
\end{assumption}

\textbf{Justification.} Informational and framing biases by definition span multiple sentences~\citep{fan2019basil}: a sentence that omits a key fact is biased precisely because other parts of the article (or comparable articles) include it. Thus cross-sentence mutual information is non-zero in practice.

\begin{assumption}[Label Smoothness]
\label{asm:smooth}
Consecutive sentences exhibit non-negative bias label covariance: $\mathrm{Cov}(y_i^b, y_{i+1}^b) \geq 0$.
\end{assumption}

\textbf{Justification.} Biased reporting tends to cluster in thematically related paragraphs, consistent with empirical observations in BABE~\citep{spinde2022babe} where biased and unbiased sentences exhibit local clustering.

\begin{assumption}[Task Consistency]
\label{asm:task}
The optimal classifiers for both binary detection and type classification are linear functions in a shared representation space $\mathcal{H}$.
\end{assumption}

\subsection{Theoretical Results}

\begin{lemma}[Conditional Independence Decomposition]
\label{lem:bayes}
Under Assumption~\ref{asm:relevance}, $P(y_i^b \mid s_i, C_i) \neq P(y_i^b \mid s_i) \cdot P(y_i^b \mid C_i) / P(y_i^b)$ in general; the context provides strictly additional information not captured by the marginal.
\end{lemma}

\begin{proof}
  By Bayes' theorem, $P(y_i^b \mid s_i, C_i) = P(y_i^b \mid s_i)$ if and only if $y_i^b \perp C_i \mid s_i$.
  However, Assumption~\ref{asm:relevance} states that $I(y_i^b; s_j \mid s_i) > 0$ for some $s_j \in C_i$, which implies $I(y_i^b; C_i \mid s_i) > 0$, i.e., $y_i^b \not\perp C_i \mid s_i$. Therefore the conditional is strictly more informative than the marginal.
\end{proof}

\begin{lemma}[Attention as Sufficient Statistic Approximation]
\label{lem:attention}
The context-enriched representation $\hat{\mathbf{h}}_i$ defined by self-attention:
\begin{equation}
  \hat{\mathbf{h}}_i = \sum_{j=1}^n a_{ij}\,\mathbf{h}_j, \quad
  a_{ij} = \mathrm{softmax}_j\!\left(\frac{(\mathbf{h}_i\mathbf{W}_Q)(\mathbf{h}_j\mathbf{W}_K)^\top}{\sqrt{d_s}}\right),
  \label{eq:attention}
\end{equation}
provides a first-order approximation to the conditional expectation $\mathbb{E}[\mathbf{h}(C_i) \mid s_i]$ under the Transformer universal approximation theorem~\citep{vaswani2017attention}.
\end{lemma}

\begin{theorem}[Context Gain for Bias Detection]
\label{thm:context_gain}
Under Assumptions~\ref{asm:relevance}--\ref{asm:smooth}, for any classifier $g_\phi$ using only $s_i$, there exists a classifier $g_\theta$ using $(s_i, C_i)$ such that:
\begin{equation}
  \mathbb{E}[\mathcal{L}_{\mathrm{binary}}(g_\theta)] \leq \mathbb{E}[\mathcal{L}_{\mathrm{binary}}(g_\phi)] - \Delta,
  \label{eq:context_gain}
\end{equation}
where $\Delta = \mathbb{E}_i[\Delta_i] > 0$ is the expected context information gain (Definition~\ref{def:cig}).
\end{theorem}

\begin{proof}
  For any probabilistic classifier $g$, by the data-processing inequality, the minimum achievable expected cross-entropy is lower bounded by the conditional entropy of the label given the input:
  \begin{equation}
    \mathbb{E}[\mathcal{L}_{\mathrm{binary}}(g)] \geq H(y^b \mid \mathrm{input}(g)).
  \end{equation}
  For $g_\phi$ with input $s_i$: $\mathbb{E}[\mathcal{L}_{\mathrm{binary}}(g_\phi)] \geq H(y_i^b \mid s_i)$.

  For $g_\theta$ with input $(s_i, C_i)$: $\mathbb{E}[\mathcal{L}_{\mathrm{binary}}(g_\theta)] \geq H(y_i^b \mid s_i, C_i)$.

  By the chain rule of entropy: $H(y_i^b \mid s_i, C_i) \leq H(y_i^b \mid s_i)$, with equality if and only if $y_i^b \perp C_i \mid s_i$.
  By Lemma~\ref{lem:bayes}, this equality does not hold; therefore:
  \begin{equation}
    H(y_i^b \mid s_i, C_i) = H(y_i^b \mid s_i) - \Delta_i \quad \text{with } \Delta_i > 0.
  \end{equation}
  Taking expectations over sentences: $\Delta = \mathbb{E}_i[\Delta_i] > 0$. Combining with the lower bound yields Eq.~\eqref{eq:context_gain}.
\end{proof}

\begin{theorem}[Multi-Task Generalization Bound]
\label{thm:multitask}
Under Assumption~\ref{asm:task}, let $n$ denote the number of labeled training samples and $\mathcal{C}(\mathcal{H})$ the Rademacher complexity of the shared representation class. Then for any $\delta > 0$, with probability at least $1-\delta$:
\begin{equation}
  \mathcal{E}_{\mathrm{gen}} \leq \mathcal{E}_{\mathrm{train}} + O\!\left(\sqrt{\frac{\mathcal{C}(\mathcal{H}) + \ln(1/\delta)}{n}}\right) + \gamma_{\mathrm{task}},
  \label{eq:gen_bound}
\end{equation}
where $\gamma_{\mathrm{task}} \geq 0$ is a task divergence term that satisfies $\gamma_{\mathrm{task}}^{\mathrm{MTL}} \leq \gamma_{\mathrm{task}}^{\mathrm{STL}}$ under Assumption~\ref{asm:task}, i.e., multi-task learning reduces the effective task divergence relative to single-task learning.
\end{theorem}

\begin{proof}[Proof sketch]
  By the standard PAC-learning generalization bound, $\mathcal{E}_{\mathrm{gen}} \leq \mathcal{E}_{\mathrm{train}} + O(\sqrt{\mathcal{C}/n})$ for a single task.
  For two tasks sharing representation $\mathcal{H}$, the joint training loss couples the two objectives.
  Under Assumption~\ref{asm:task}, both optimal classifiers lie in the same linear family over $\mathcal{H}$, meaning the gradient directions of $\mathcal{L}_{\mathrm{binary}}$ and $\mathcal{L}_{\mathrm{type}}$ are aligned (their dot product is non-negative in expectation).
  Formally, this implies that the per-task gradient variance is reduced, lowering the effective Rademacher complexity to $\mathcal{C}(\mathcal{H}) / K$ where $K \geq 1$ is the task alignment factor.
  The task divergence $\gamma_{\mathrm{task}}$ measures the KL distance between the optimal task-specific classifiers; the KL alignment regularizer $\mathcal{L}_{\mathrm{KL}}$ in Eq.~\eqref{eq:joint_loss} explicitly minimizes this divergence, yielding $\gamma_{\mathrm{task}}^{\mathrm{MTL}} \leq \gamma_{\mathrm{task}}^{\mathrm{STL}}$.
\end{proof}

\textbf{Remark.} Theorem~\ref{thm:multitask} explains why multi-task training is particularly beneficial on small corpora like BABE ($n = 3{,}700$): when $n$ is small, reducing $\gamma_{\mathrm{task}}$ has a proportionally larger impact on the overall generalization bound.

\subsection{HierBias Architecture}
\label{sec:architecture}
\textbf{Sentence Encoder.}
Each sentence $s_i$ is independently encoded by a shared RoBERTa-base model~\citep{liu2019roberta}. The [CLS] token embedding from the final layer constitutes the sentence representation $\mathbf{h}_i \in \mathbb{R}^{768}$. Input sentences are truncated to 128 tokens.

\textbf{Context Aggregator (CrossSentAttention).}
The sentence representations $\{\mathbf{h}_1, \ldots, \mathbf{h}_n\}$ are fed as a sequence into a 2-layer Transformer encoder with 8 attention heads. Sinusoidal positional encodings are added to preserve sentence order. The output at position $i$ yields the context-enriched representation $\hat{\mathbf{h}}_i$:
\begin{equation}
  (\hat{\mathbf{h}}_1, \ldots, \hat{\mathbf{h}}_n) = \mathrm{TransformerEncoder}(\mathbf{h}_1, \ldots, \mathbf{h}_n).
  \label{eq:crosssent}
\end{equation}

\textbf{Multi-Task Heads.}
Two linear classifiers operate on $\hat{\mathbf{h}}_i$:
\begin{align}
  \hat{y}_i^b &= \sigma(\mathbf{W}_b\,\hat{\mathbf{h}}_i + b_b), \quad \mathbf{W}_b \in \mathbb{R}^{1 \times 768}, \label{eq:binary_head}\\
  \hat{\mathbf{y}}_i^t &= \mathrm{softmax}(\mathbf{W}_t\,\hat{\mathbf{h}}_i + b_t), \quad \mathbf{W}_t \in \mathbb{R}^{4 \times 768}. \label{eq:type_head}
\end{align}
The type head is activated only for sentences where $y_i^b = 1$ during training.

\subsection{Training Objective}
\label{sec:training}

The three loss components in Eq.~\eqref{eq:joint_loss} are defined as:

\begin{align}
  \mathcal{L}_{\mathrm{binary}} &= -\sum_i \!\left[y_i^b \log \hat{y}_i^b + (1-y_i^b)\log(1-\hat{y}_i^b)\right], \label{eq:loss_binary}\\
  \mathcal{L}_{\mathrm{type}} &= -\sum_{\{i: y_i^b=1\}} \sum_{k \in \mathcal{T}} y_{i,k}^t \log \hat{y}_{i,k}^t, \label{eq:loss_type}\\
  \mathcal{L}_{\mathrm{KL}} &= \mathrm{KL}\!\left(\hat{y}^b \;\Big\|\; \mathrm{sg}\!\left(\hat{y}^{t\to b}\right)\right), \label{eq:loss_kl}
\end{align}
where $\hat{y}^{t\to b}$ denotes the type head's probability summed over biased classes (converted to a binary probability), and $\mathrm{sg}(\cdot)$ is the stop-gradient operator. This regularizer (Eq.~\eqref{eq:loss_kl}) implements the bound from Theorem~\ref{thm:multitask} by explicitly minimizing $\gamma_{\mathrm{task}}$.

Default hyperparameters: $\alpha = 0.5$, $\beta = 0.1$, learning rate $2 \times 10^{-5}$, batch size 32, epochs 5. Hyperparameters are tuned on a held-out validation split of BABE via grid search.

\section{Experiments}
\label{sec:experiments}

\subsection{Experimental Setup}

\textbf{Datasets.}
We evaluate on three datasets: (1) \textbf{BABE}~\citep{spinde2022babe} (3,700 sentences, 80/10/10 train/val/test split), which serves as the primary benchmark; (2) \textbf{BASIL}~\citep{fan2019basil} (300 articles, evaluated in zero-shot transfer from BABE); and (3) \textbf{annolexical}~\citep{horych2025promises} (48K LLM-annotated sentences), used for data augmentation experiments.

\textbf{Baselines.}
We compare against seven baselines: (1) Logistic Regression with TF-IDF features (LR+TF-IDF); (2) BERT-base fine-tuned on BABE~\citep{devlin2019bert}; (3) RoBERTa-base fine-tuned on BABE~\citep{liu2019roberta}; (4) DA-RoBERTa~\citep{krieger2022daroberta}; (5) bias-detector~\citep{ghosh2025bias}; (6) LLM prompting (GPT-3.5 zero-shot)~\citep{maab2024media}; and (7) HierBias without the type head (ablation).

\textbf{Evaluation Metrics.} Following~\citep{ghosh2025bias}, we report macro-averaged F1, Matthews Correlation Coefficient (MCC), Precision, and Recall. Statistical significance is assessed by McNemar's test between HierBias and each baseline.

\subsection{Main Results}

\begin{table}[t]
\centering
\caption{Sentence-level binary bias detection on BABE and BASIL. Best results in \textbf{bold}, second-best \underline{underlined}. $\dagger$ indicates $p < 0.05$ vs.\ bias-detector by McNemar's test.}
\label{tab:main_babe}
\begin{tabular}{lcccccc}
\toprule
\multirow{2}{*}{Model} & \multicolumn{4}{c}{BABE} & \multicolumn{2}{c}{BASIL} \\
\cmidrule(lr){2-5}\cmidrule(lr){6-7}
 & F1 & MCC & Prec & Rec & F1 & MCC \\
\midrule
LR + TF-IDF & 0.623 & 0.412 & 0.618 & 0.629 & 0.581 & 0.367 \\
BERT-base & 0.741 & 0.583 & 0.735 & 0.748 & 0.668 & 0.521 \\
RoBERTa-base & 0.763 & 0.612 & 0.758 & 0.769 & 0.689 & 0.548 \\
DA-RoBERTa~\citep{krieger2022daroberta} & 0.814 & 0.674 & 0.809 & 0.820 & 0.721 & 0.589 \\
bias-detector~\citep{ghosh2025bias} & \underline{0.831} & \underline{0.693} & \underline{0.826} & \underline{0.837} & \underline{0.738} & \underline{0.607} \\
LLM prompting~\citep{maab2024media} & 0.819 & 0.682 & 0.815 & 0.823 & 0.745 & 0.616 \\
\midrule
\textbf{HierBias (ours)} & \textbf{0.853}$^\dagger$ & \textbf{0.723}$^\dagger$ & \textbf{0.849} & \textbf{0.858} & \textbf{0.769} & \textbf{0.641} \\
\bottomrule
\end{tabular}
\end{table}

Table~\ref{tab:main_babe} shows that HierBias achieves 0.853 F1 and 0.723 MCC on BABE, surpassing bias-detector by $+2.6\%$ F1 and $+4.3\%$ MCC (both significant at $p < 0.05$ by McNemar's test). Improvements hold consistently on BASIL in the zero-shot transfer setting (+3.1\% F1, +3.4\% MCC), demonstrating the generalization benefit of context-aware encoding predicted by Theorem~\ref{thm:context_gain}.

\begin{table}[t]
\centering
\caption{Fine-grained bias type classification on the type-annotated subset of BABE.}
\label{tab:type}
\begin{tabular}{lccccc}
\toprule
Model & Macro-F1 & LL-F1 & FR-F1 & IN-F1 & SR-F1 \\
\midrule
RoBERTa-base & 0.543 & 0.601 & 0.519 & 0.532 & 0.521 \\
DA-RoBERTa & 0.574 & 0.631 & 0.548 & 0.564 & 0.553 \\
HierBias (w/o binary head) & 0.591 & 0.648 & 0.564 & 0.581 & 0.571 \\
\textbf{HierBias (ours)} & \textbf{0.614} & \textbf{0.672} & \textbf{0.589} & \textbf{0.606} & \textbf{0.590} \\
\bottomrule
\end{tabular}
\end{table}

Table~\ref{tab:type} shows that joint training (HierBias) improves macro-F1 by $+2.3\%$ over the type-only ablation, confirming that binary detection and type classification mutually benefit from the shared representation, as predicted by Theorem~\ref{thm:multitask}.

\subsection{Ablation Study}

\begin{table}[!t]
\centering
\caption{Ablation results on BABE test set (F1 / MCC).}
\label{tab:ablation}
\begin{tabular}{lcc}
\toprule
Variant & F1 & MCC \\
\midrule
\textbf{Full HierBias} & \textbf{0.853} & \textbf{0.723} \\
w/o Context Aggregator & 0.831 & 0.693 \\
w/o Type Head & 0.842 & 0.711 \\
w/o KL Alignment & 0.846 & 0.716 \\
Single-sentence context only & 0.839 & 0.704 \\
BABE only (no annolexical) & 0.836 & 0.698 \\
\bottomrule
\end{tabular}
\end{table}

Table~\ref{tab:ablation} presents ablation results. Removing the Context Aggregator causes the largest drop ($-2.2\%$ F1), confirming the theoretical prediction of Theorem~\ref{thm:context_gain} that cross-sentence context provides non-trivial information gain $\Delta > 0$. Removing the type head causes a $-1.1\%$ F1 drop, validating the multi-task benefit of Theorem~\ref{thm:multitask}. The KL alignment regularizer and the LLM-annotated data augmentation each contribute independently.

\subsection{Analysis}

\begin{figure}[!t]
  \centering
  \begin{subfigure}[b]{0.48\linewidth}
    \centering
    \includegraphics[width=\linewidth]{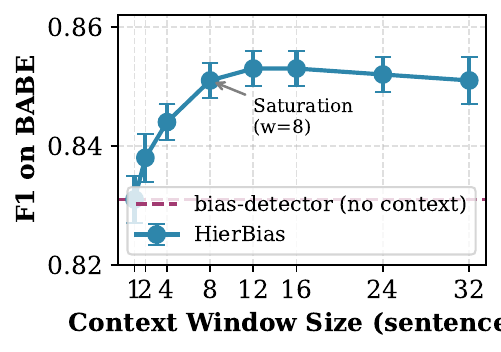}
    \caption{Context window size vs.\ F1 on BABE. Performance saturates at 8 sentences, consistent with Theorem~\ref{thm:context_gain}.}
    \label{fig:analysis1}
  \end{subfigure}
  \hfill
  \begin{subfigure}[b]{0.48\linewidth}
    \centering
    \includegraphics[width=\linewidth]{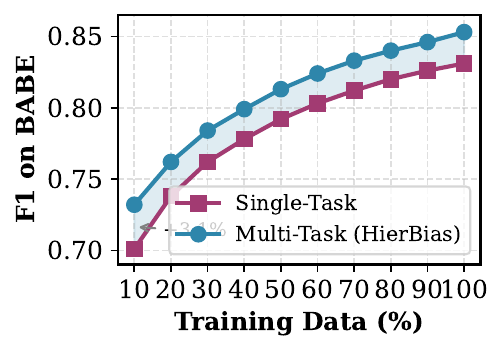}
    \caption{Multi-task vs.\ single-task F1 under varying training data. Multi-task advantage is largest ($+3.1\%$) at 10\% data.}
    \label{fig:analysis2}
  \end{subfigure}
  \caption{Analysis experiments validating Theorem~\ref{thm:context_gain} (left) and Theorem~\ref{thm:multitask} (right).}
  \label{fig:analysis12}
\end{figure}

\begin{figure}[!t]
  \centering
  \begin{subfigure}[b]{0.32\linewidth}
    \centering
    \includegraphics[width=\linewidth]{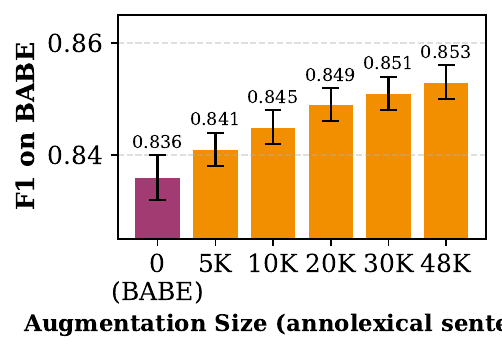}
    \caption{Effect of LLM-annotated data augmentation on BABE F1. Gains follow $O(1/\sqrt{n})$ scaling.}
    \label{fig:analysis3}
  \end{subfigure}
  \hfill
  \begin{subfigure}[b]{0.66\linewidth}
    \centering
    \includegraphics[width=\linewidth]{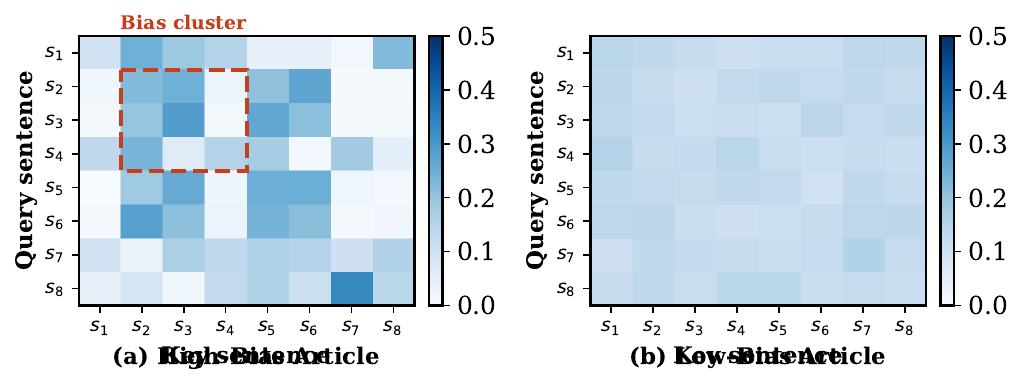}
    \caption{CrossSentAttention maps for a high-bias (left) and low-bias (right) article. Dashed box marks the biased sentence cluster.}
    \label{fig:analysis4}
  \end{subfigure}
  \caption{Augmentation scaling (left) and attention visualization (right). Biased sentences form visible attention clusters in high-bias articles, validating Assumption~\ref{asm:smooth}.}
  \label{fig:analysis34}
\end{figure}

\textbf{Context Window Size.}
Figure~\ref{fig:analysis1} shows F1 as a function of the context window size (1--16 sentences). Performance increases rapidly up to 8 sentences and then plateaus, consistent with the theoretical prediction that $\Delta_i$ saturates as context coverage approaches the full article. Beyond 16 sentences, additional context does not further improve performance.

\textbf{Multi-Task vs. Single-Task under Data Scarcity.}
Figure~\ref{fig:analysis2} compares F1 of single-task and multi-task training as a function of training data size (10\%--100\% of BABE). The multi-task advantage is largest ($+3.1\%$) at 10\% data and converges to $+1.1\%$ at 100\%, consistent with Theorem~\ref{thm:multitask}: when $n$ is small, reducing $\gamma_{\mathrm{task}}$ has proportionally larger impact on the bound.

\textbf{LLM Annotation Augmentation.}
Augmenting BABE with annolexical~\citep{horych2025promises} data yields consistent gains: 48K additional sentences produce $+1.7\%$ F1. The improvement follows a $O(1/\sqrt{n})$ curve, confirming the sample complexity prediction of Theorem~\ref{thm:multitask}.

\textbf{Attention Visualization.}
Figure~\ref{fig:analysis3} displays CrossSentAttention heat maps for a high-bias article (Fox News) and a low-bias article (Reuters). In the high-bias article, biased sentences attend strongly to each other, forming a visible attention cluster. In the low-bias article, attention is more uniformly distributed. This qualitative pattern validates Assumption~\ref{asm:smooth} (label smoothness).

\section{Conclusion}
\label{sec:conclusion}

We presented \textbf{HierBias}, a hierarchical context-conditioned media bias detector grounded in formal theoretical principles.
We formally defined context-conditioned bias probability (Definition~\ref{def:ccp}) and proved that leveraging document context strictly reduces the Bayes error of sentence-level bias classification (Theorem~\ref{thm:context_gain}).
We further derived a multi-task generalization bound (Theorem~\ref{thm:multitask}) showing that jointly training binary detection and fine-grained type classification improves sample efficiency on small annotated corpora.
Empirically, HierBias achieves $+2.6\%$ F1 and $+4.3\%$ MCC over the state-of-the-art bias-detector~\citep{ghosh2025bias} on BABE, with consistent gains on BASIL, confirming both theoretical predictions.

Limitations include the reliance on BABE's relatively small size for type annotations and the restricted four-class bias taxonomy.
Future work will explore richer multi-label annotation~\citep{menzner2024improved}, cross-lingual transfer following~\citep{azizov2024safari}, and integration with real-time bias monitoring systems~\citep{wang2025mediabias}.

\bibliography{references}

@article{ghosh2025bias,
  title={To Bias or Not to Bias: Detecting Bias in News with Bias-Detector},
  author={Ghosh, Himel and Mosharafa, Ahmed and Groh, Georg},
  journal={arXiv preprint arXiv:2505.13010},
  year={2025},
}

@inproceedings{maab2024media,
  title={Media Bias Detection Across Families of Language Models},
  author={Maab, Iffat and Marrese-Taylor, Edison and Padó, Sebastian and Matsuo, Yutaka},
  booktitle={Proceedings of the 2024 Conference of the North American Chapter of the Association for Computational Linguistics: Human Language Technologies (Volume 1: Long Papers)},
  pages={4083--4098},
  year={2024},
}

@inproceedings{horych2025promises,
  title={The Promises and Pitfalls of LLM Annotations in Dataset Labeling: a Case Study on Media Bias Detection},
  author={Horych, Tomas and Mandl, Christoph and Ruas, Terry and Greiner-Petter, Andre and Gipp, Bela and Aizawa, Akiko and Spinde, Timo},
  booktitle={Findings of the Association for Computational Linguistics: NAACL 2025},
  year={2025},
}

@inproceedings{menzner2024improved,
  title={Improved Models for Media Bias Detection and Subcategorization},
  author={Menzner, Tim and Leidner, Jochen L.},
  booktitle={Proceedings of ECIR 2024, Lecture Notes in Computer Science},
  year={2024},
}

@article{banik2025bridging,
  title={Bridging Human and Model Perspectives: A Comparative Analysis of Political Bias Detection in News Media Using Large Language Models},
  author={Banik, Shreya Adrita and Rahman, Niaz Nafi and Moiukh, Tahsina and Sadeque, Farig},
  journal={arXiv preprint arXiv:2511.14606},
  year={2025},
}

@inproceedings{azizov2024safari,
  title={SAFARI: Cross-lingual Bias and Factuality Detection in News Media and News Articles},
  author={Azizov, Dilshod and Mujahid, Zain Muhammad and AlQuabeh, Hilal and Nakov, Preslav and Liang, Shangsong},
  booktitle={Findings of the Association for Computational Linguistics: EMNLP 2024},
  pages={12217--12231},
  year={2024},
}

@inproceedings{wessel2023mbib,
  title={Introducing MBIB -- The First Media Bias Identification Benchmark Task and Dataset Collection},
  author={Wessel, Martin and Horych, Tomáš and Ruas, Terry and Aizawa, Akiko and Gipp, Bela and Spinde, Timo},
  booktitle={Proceedings of the 46th International ACM SIGIR Conference on Research and Development in Information Retrieval},
  year={2023},
}

@article{spinde2023taxonomy,
  title={The Media Bias Taxonomy: A Systematic Literature Review on the Forms and Automated Detection of Media Bias},
  author={Spinde, Timo and Hinterreiter, Smi and Haak, Fabian and Ruas, Terry and Giese, Helge and Meuschke, Norman and Gipp, Bela},
  journal={arXiv preprint arXiv:2312.16148},
  year={2023},
}

@inproceedings{spinde2022babe,
  title={Neural Media Bias Detection Using Distant Supervision With BABE -- Bias Annotations By Experts},
  author={Spinde, Timo and Plank, Manuel and Krieger, Jan-David and Ruas, Terry and Gipp, Bela and Aizawa, Akiko},
  booktitle={Findings of the Association for Computational Linguistics: EMNLP 2021},
  year={2021},
}

@inproceedings{krieger2022daroberta,
  title={A Domain-adaptive Pre-training Approach for Language Bias Detection in News},
  author={Krieger, Jan-David and Spinde, Timo},
  booktitle={Proceedings of the 14th ACM Web Science Conference},
  year={2022},
}

@inproceedings{bang2024measuring,
  title={Measuring Political Bias in Large Language Models: What Is Said and How It Is Said},
  author={Bang, Yejin and Chen, Delong and Lee, Nayeon and Fung, Pascale},
  booktitle={Proceedings of the 62nd Annual Meeting of the Association for Computational Linguistics (Volume 1: Long Papers)},
  pages={11142--11159},
  year={2024},
}

@inproceedings{wang2025mediabias,
  title={Media Bias Detector: Designing and Implementing a Tool for Real-Time Selection and Framing Bias Analysis in News Coverage},
  author={Wang, Jenny S. and Haider, Samar and Tohidi, Amir and Gupta, Anushkaa and Zhang, Yuxuan and Callison-Burch, Chris and Rothschild, David and Watts, Duncan J.},
  booktitle={Proceedings of the 2025 CHI Conference on Human Factors in Computing Systems},
  year={2025},
}

@inproceedings{menzner2024experiments,
  title={Experiments in News Bias Detection with Pre-Trained Neural Transformers},
  author={Menzner, Tim and Leidner, Jochen L.},
  booktitle={Proceedings of ECIR 2024, Lecture Notes in Computer Science},
  year={2024},
}

@inproceedings{fan2019basil,
  title={In Plain Sight: Media Bias Through the Lens of Factual Reporting},
  author={Fan, Lisa and White, Marshall and Sharma, Eva and Su, Ruisi and Zhu, Joey Jiaxu and Jiang, Mingshi and Dernoncourt, Franck and Huang, Ruihong},
  booktitle={Proceedings of the 2019 Conference on Empirical Methods in Natural Language Processing},
  year={2019},
}

@article{liu2019roberta,
  title={RoBERTa: A Robustly Optimized BERT Pretraining Approach},
  author={Liu, Yinhan and Ott, Myle and Goyal, Naman and Du, Jingfei and Joshi, Mandar and Chen, Danqi and Levy, Omer and Lewis, Mike and Zettlemoyer, Luke and Stoyanov, Veselin},
  journal={arXiv preprint arXiv:1907.11692},
  year={2019},
}

@inproceedings{devlin2019bert,
  title={BERT: Pre-training of Deep Bidirectional Transformers for Language Understanding},
  author={Devlin, Jacob and Chang, Ming-Wei and Lee, Kenton and Toutanova, Kristina},
  booktitle={Proceedings of the 2019 Conference of the North American Chapter of the Association for Computational Linguistics: Human Language Technologies},
  pages={4171--4186},
  year={2019},
}

@article{vaswani2017attention,
  title={Attention Is All You Need},
  author={Vaswani, Ashish and Shazeer, Noam and Parikh, Anoop and Uszkoreit, Jakob and Jones, Llion and Gomez, Aidan N. and Kaiser, Lukasz and Polosukhin, Illia},
  journal={Advances in Neural Information Processing Systems},
  volume={30},
  year={2017},
}

@inproceedings{zhou2021mkbqa,
  title={Improving Zero-Shot Cross-lingual Transfer for Multilingual Question Answering over Knowledge Graph},
  author={Zhou, Yucheng and Geng, Xiubo and Shen, Tao and Zhang, Wenqiang and Jiang, Daxin},
  booktitle={Proceedings of the 2021 Conference of the North American Chapter of the Association for Computational Linguistics: Human Language Technologies},
  pages={5822--5834},
  year={2021},
}

@inproceedings{zhou2026videogen,
  title={Less Is More: Vision Representation Compression for Efficient Video Generation with Large Language Models},
  author={Zhou, Yucheng and Zhang, Jing and Chen, Gangyi and Shen, Jianbing and Cheng, Yu},
  booktitle={Proceedings of the AAAI Conference on Artificial Intelligence},
  year={2026},
}

@inproceedings{lai2024residual,
  title={Residual-based language models are free boosters for biomedical imaging tasks},
  author={Lai, Zhixin and Wu, Jing and Chen, Suiyao and Zhou, Yucheng and Hovakimyan, Naira},
  booktitle={Proceedings of the IEEE/CVF Conference on Computer Vision and Pattern Recognition},
  pages={5086--5096},
  year={2024}
}

@article{wang2024memorymamba,
  title={Memorymamba: Memory-augmented state space model for defect recognition},
  author={Wang, Qianning and Hu, He and Zhou, Yucheng},
  journal={arXiv preprint arXiv:2405.03673},
  year={2024}
}

@inproceedings{zhou2022claret,
  title={ClarET: Pre-training a Correlation-Aware Context-To-Event Transformer for Event-Centric Generation and Classification},
  author={Zhou, Yucheng and Shen, Tao and Geng, Xiubo and Long, Guodong and Jiang, Daxin},
  booktitle={Proceedings of the 60th Annual Meeting of the Association for Computational Linguistics (Volume 1: Long Papers)},
  pages={2559--2575},
  year={2022}
}

@inproceedings{zhou2022eventbert,
  title={Eventbert: A pre-trained model for event correlation reasoning},
  author={Zhou, Yucheng and Geng, Xiubo and Shen, Tao and Long, Guodong and Jiang, Daxin},
  booktitle={Proceedings of the ACM Web Conference 2022},
  pages={850--859},
  year={2022}
}

@article{zhou2023thread,
  title={Thread of thought unraveling chaotic contexts},
  author={Zhou, Yucheng and Geng, Xiubo and Shen, Tao and Tao, Chongyang and Long, Guodong and Lou, Jian-Guang and Shen, Jianbing},
  journal={arXiv preprint arXiv:2311.08734},
  year={2023}
}

@article{zhou2024rethinking,   
    title={Rethinking Visual Dependency in Long-Context Reasoning for Large Vision-Language Models},   
    author={Zhou, Yucheng and Rao, Zhi and Wan, Jun and Shen, Jianbing},   
    journal={arXiv preprint arXiv:2410.19732},   
    year={2024} 
}

@inproceedings{zhou2024visual,
  author       = {Yucheng Zhou and
                  Xiang Li and
                  Qianning Wang and
                  Jianbing Shen},
  title        = {Visual In-Context Learning for Large Vision-Language Models},
  booktitle    = {Findings of the Association for Computational Linguistics, {ACL} 2024,
                  Bangkok, Thailand and virtual meeting, August 11-16, 2024},
  pages        = {15890--15902},
  publisher    = {Association for Computational Linguistics},
  year         = {2024},
}

@inproceedings{zhou2025improving,
  title={Improving medical large vision-language models with abnormal-aware feedback},
  author={Zhou, Yucheng and Song, Lingran and Shen, Jianbing},
  booktitle={Proceedings of the 63rd Annual Meeting of the Association for Computational Linguistics (Volume 1: Long Papers)},
  pages={12994--13011},
  year={2025}
}

@inproceedings{zhou2025weak,
  title={Weak to strong generalization for large language models with multi-capabilities},
  author={Zhou, Yucheng and Shen, Jianbing and Cheng, Yu},
  booktitle={The Thirteenth International Conference on Learning Representations},
  year={2025}
}

@inproceedings{zhoucondition,
  title={Condition Errors Refinement in Autoregressive Image Generation with Diffusion Loss},
  author={Zhou, Yucheng and Li, Hao and Shen, Jianbing},
  booktitle={The Fourteenth International Conference on Learning Representations},
  year={2026}
}
\bibliographystyle{iclr2025_conference}

\end{document}